\documentclass[letterpaper, 10 pt, conference]{ieeeconf}
\IEEEoverridecommandlockouts
\usepackage{cite}      
\usepackage{amsmath}   
\usepackage{amssymb}   

\usepackage{amsthm}   
\usepackage{graphicx}  
\usepackage{float}     
\usepackage{subfig}    
\usepackage{algorithm}
\usepackage{algorithmic}
\usepackage{bbm}
\usepackage{svg}       
\usepackage[most]{tcolorbox}   

\usepackage{microtype}
\usepackage{booktabs}

\usepackage{hyperref}    
\usepackage[capitalize,noabbrev]{cleveref}
\crefname{theorem}{Theorem}{Theorems}
\crefname{lemma}{Lemma}{Lemmas}
\crefname{defsec}{Definition}{Definitions}
\crefname{subsection}{Definition}{Definitions}
\crefname{definition}{Definition}{Definitions}

\usepackage{xcolor}
\definecolor{objzero}{RGB}{255,140,0} 
\definecolor{objone}{RGB}{30,144,255} 
\newcommand{\fzero}{\textcolor{objzero}{f_0}}
\newcommand{\fone}{\textcolor{objone}{f_1}}

\tcolorboxenvironment{theorem}{
  colback=white,
  colframe=black,
  boxrule=0.5pt,
  arc=0pt,
  left=5pt,
  right=5pt,
  top=5pt,
  bottom=5pt
}

\newtheorem{theorem}{Theorem}[section]

\newtheorem{lemma}[theorem]{Lemma}

\usepackage[textsize=tiny]{todonotes}
\usepackage{xspace}

\usepackage{amsmath,amsfonts,bm}

\def\eqref#1{equation~\ref{#1}}

\def\1{\bm{1}}

\DeclareMathAlphabet{\mathsfit}{\encodingdefault}{\sfdefault}{m}{sl}
\SetMathAlphabet{\mathsfit}{bold}{\encodingdefault}{\sfdefault}{bx}{n}

\newcommand{\pmean}[1]{\texorpdfstring{\ensuremath{\overline{\mu}_{#1}}}{pmean}}

\newcommand{\FQvalue}{\ensuremath{\mathbf{FQ}}}
\newcommand{\FVobs}{\ensuremath{\mathbf{FV}^{\text{obs}}}}

\newcommand{\LTD}{\ensuremath{L^{\text{TD}}}}
\newcommand{\LFV}{\ensuremath{L^{\text{FV}}}}
\newcommand{\vecand}[2]{\ensuremath{(#2)}^{\land#1}}

\newcommand{\pand}[1]{\ensuremath{\land^{#1}}}
\newcommand{\por}[1]{\ensuremath{\lor^{#1}}}

\newcommand{\alg}{\textbf{BPG}\xspace}
\newcommand{\algFull}{Balanced Policy Gradient\xspace}
\newcommand{\lang}{\textbf{FPL}\xspace}

\newcommand{\langFull}{Fulfillment Priority Logic\xspace}

\newcommand{\gap}{\emph{intent-to-behavior gap}\xspace}

\newcommand{\ourMDP}{MF-MDP\xspace}

\newcommand{\inlinefig}[1]{%
  \raisebox{-0.2\height}{\includegraphics[height=1.2em]{#1}}%
}

\newcommand{\vect}[1]{\vec{\mathbf{#1}}}

\title{\LARGE \bf Closing the \gap via \langFull}

\author{
    B.~El~Mabsout\textsuperscript{*}\textsuperscript{1}, 
    A.~Abdelgawad\textsuperscript{*}\textsuperscript{2}, and 
    R.~Mancuso\textsuperscript{1}
    \thanks{*These authors contributed equally to this work.}
    \thanks{\textsuperscript{1}Department of Computer Science, Boston University, Boston, MA, USA}
    \thanks{\textsuperscript{2}Systems Engineering Division, Boston University, Boston, MA, USA}
}

\usepackage{array}  

\begin{document}

\maketitle
  \thispagestyle{empty}
\pagestyle{empty}

\begin{abstract}
Practitioners designing reinforcement learning policies face a fundamental challenge: translating intended behavioral objectives into representative reward functions.
This challenge stems from behavioral intent requiring simultaneous achievement of multiple competing objectives, typically addressed through labor-intensive linear reward composition that yields brittle results.
Consider the ubiquitous robotics scenario where performance maximization directly conflicts with energy conservation. Such competitive dynamics are resistant to simple linear reward combinations.
In this paper, we present the concept of objective fulfillment upon which we build \langFull (\lang). \lang allows practioners to define logical formula representing their intentions and priorities within multi-objective reinforcement learning.
Our novel \algFull algorithm leverages \lang specifications to achieve up to 500\% better sample efficiency compared to Soft Actor Critic.
Notably, this work constitutes the first implementation of non-linear utility scalarization design, specifically for continuous control problems.
\end{abstract}

\section{Introduction}

Reward design in reinforcement learning is a nuanced and intricate process that presents the complex challenge of aligning agent behavior with intended objectives \cite{comprehensive_reward_eng_and_shaping}.
Recent work by~\cite{booth2023perils}, supported by findings from~\cite{KNOX2023103829}, has exposed significant limitations in current practices, revealing that 92\% of surveyed RL experts rely on trial-and-error approaches, leading to overfitted and inadequate reward functions.
This issue fundamentally stems from the disconnect between human reward conceptualization and RL optimization mechanisms~\cite{booth2023perils}, creating what we term the \gap—the disparity between practitioners' intended behavioral objectives and the actual behaviors exhibited by policies after optimization.
We argue that traditional trial-and-error reward weight tuning in response to observed policy behaviors \cite{hayes2023brief,Limitations_of_Scalarisation} is inherently flawed and propose a systematic solution.

Researchers have approached this \gap from various perspectives.
Formal methods researchers adopted structured specifications such as temporal logics forming reward signals, providing a principled framework to express temporal controller behavior precisely~\cite{Belta_Temporal}. 
Others focused on reward engineering, exploiting their domain knowledge as demonstrated by \cite{lee2020learning} in their quadrupedal locomotion work. They meticulously crafted a multi-component reward function (combining linear velocity, angular velocity, base motion stability, foot clearance, collision avoidance, trajectory smoothness, and torque minimization) with carefully balanced weights to achieve robust locomotion over challenging terrain.
Other methods avoid this troublesome process altogether through inverse RL~\cite{inverse_rl_survey}, while more recent work delegates the whole reward design problem to Large Language Models \cite{eureka,yu2023language}.

Recognizing scalar rewards' inherent limitations for expressing multifaceted intentions, multi-objective reinforcement learning (MORL) emerged to represent distinct objectives through vector rewards~\cite{survey_seq_dec_morl,reymond2023actor}.
While there is support for Sutton's reward hypothesis suggesting that maximization of any goals can theoretically reduce to maximizing a scalar signal~\cite{sutton2018reinforcement, reward_is_enough, survey_seq_dec_morl}, later contributions~\cite{reward_hypothesis_false, settling_reward_hypothesis} show that there are cases where this does not hold.
Indeed, MORL research reveals that per-step scalarization is not expressive enough to represent desired behaviors adequately.
This explains practitioners' labor-intensive iterative reward tuning.
The challenge shifts to developing methodologies that avoid multiple design-evaluate-adjust cycles.
By maintaining vector representation throughout learning, MORL enables separate objective estimation before trade-offs.
However, comparing policies in multi-objective spaces introduces a partial ordering problem—one policy may excel in certain objectives while underperforming in others—requiring scalarization utility functions (typically weighted sums) to establish total ordering.
Yet even with this delayed scalarization, traditional linear utility functions often drive policies toward suboptimal local minima when objectives fundamentally conflict~\cite{MOMARL}.

In this paper, we introduce \langFull (\lang)—a logic transforming declarative policy behavior descriptions into utility functions through principled algebraic transformations, enabling practitioners to express objectives semantically while maintaining mathematical rigor and providing formal guarantees on objective fulfillment.
Our approach directly addresses the "Unspecified and Multi-Objective Reward Functions" challenge in real-world RL where traditional methods require complex, hand-crafted rewards~\cite{RL_challenges}.
\lang replaces weight-based specifications with logical priorities through three key innovations: (1) generalized mean operators for flexible objective composition, generalizing linear utilities, (2) Q-value level scalarization preserving intended relationships, and (3) principled normalization for stable cross-objective learning.
Implementing these advances in our novel \algFull (\alg) algorithm, we demonstrate up to 500\% improvement in sample efficiency over Soft Actor Critic.
\cref{sec:related_works} goes through the related works that tackled the \gap.
\cref{sec:definitions} provides theoretical foundations before introducing our core contributions—the \langFull (\cref{sec:aps}) and \alg algorithm (\cref{sec:bpg}). Empirical results follow in \cref{sec:experiments}.
\cref{sec:limitations} concludes by offering limitations and future directions.
\section{Related Works} \label{sec:related_works}

\subsection{Reward Design}\label{sec:erl_rdesign} Reward design presents a fundamental challenge in reinforcement learning systems~\cite{comprehensive_reward_eng_and_shaping}.
Traditional approaches integrate reward shaping techniques to improve learning efficiency~\cite{Hu2020}, yet remain constrained by the quality of underlying reward specifications.
Research has progressed along multiple trajectories: inverse reinforcement learning extracts reward functions from demonstrations~\cite{inverse_rl_survey}, while recent LLM-based methods like Eureka~\cite{eureka} leverage natural language for behavior specification, though limited by reasoning capabilities and current reward engineering methods.
Notably, complex real-world applications demonstrate an emerging pattern—practitioners inherently gravitate toward structured compositional approaches, as evidenced in tokamak reactor control~\cite{tokamak} and various multi-objective domains~\cite{ radiotherapy, pianosi2013tree, VERSTRAETEN2019428}. This aligns with our earlier geometric reward composition work~\cite{how_to_train_quad}, where we demonstrated that explicit attention to competing objectives produces more consistent policies with superior performance, even when evaluated against original environment reward functions.

\subsection{Multi-Objective RL}\label{sec:erl_morl}
Multi-objective reinforcement learning (MORL) addresses environments requiring simultaneous optimization of multiple (possibly competing) objectives. MORL approaches divide into \emph{a-priori} methods (preferences specified before training, yielding a single policy) and \emph{a-posteriori} methods (generating multiple Pareto-optimal policies for post-training selection).

Research has predominantly focused on a-posteriori approaches~\cite{SAKAWA199819, MODRL_framework}, with advances including evolutionary algorithms~\cite{xu2020prediction}, hypernet-based Pareto front approximations~\cite{shu2024learning}, and GPI-PD's~\cite{alegre2023sample} sample-efficient Convex Coverage Sets.
However, these methods typically rely on linear scalarization, which becomes limiting when objectives exhibit non-linear interactions~\cite{survey_seq_dec_morl}.
While \cite{reymond2023actor} introduced an actor-critic method that leverages a non-linear utility function, their approach assumes the utility exists and is confined to discrete action spaces. In contrast, our work explicitly designs the utility within a comprehensive framework, enabling its application to continuous control tasks.

\textbf{Our work} introduces logical operators for objective composition that better captures intended trade-offs non-linearly. To our knowledge, this represents the first comprehensive a-priori MORL framework specifically designed for continuous robot control with non-linear utility functions.

\subsection{Formal Methods}
Temporal logic frameworks~\cite{kress2009temporal, lahijanian2011temporal} provide rigorous approaches for specifying temporal robot behavior, with extensions to learning-based control~\cite{aksaray2016q} and domain-specific languages like SPECTRL~\cite{jothimurugan2019composable}.
A fundamental limitation in applying formal methods to reinforcement learning is the challenge of optimizing non-differentiable logical specifications using gradient-based methods.
The complementary nature of formal verification and priority-based optimization (\textbf{our work}) suggests potential integration pathways: formal methods could verify safety properties while priority-based approaches handle objective trade-offs, potentially addressing both correctness guarantees and optimization efficiency in complex robotics tasks.
Recent work~\cite{priority_based_temporal_logics} proposed weighted STL, employing smooth min/max and arithmetic/geometric means to handle competing specifications in control systems—all special cases of our power mean operators. While the work focuses on temporal properties, \lang extends these compositional principles to reinforcement learning with normalized objectives and priority offsets, suggesting promising integration pathways between formal methods and multi-objective RL.

Similarly, fuzzy logic approaches~\cite{fuzzy_reward_fn_rl, rl_with_fuzzy_testing} have addressed reward design through degrees of truth, creating intermediate reward landscapes that enhance learning convergence. We analyze the mathematical connections between our power mean operators and fuzzy logic
further in \cref{sec:aps_fuzzy_logic}.
\subsection{The Sample Efficiency Challenge}\label{sec:sample_efficiency}
The iterative refinement of rewards depends on both algorithmic sample efficiency and reward specification methodology.
Improvements to either accelerate achieving desired agent behavior, closing the \gap with particular importance for sample-constrained robotic applications.
Recent works have addressed the sample efficiency issue through various algorithmic and architectural strategies.
While several studies have focused on reducing estimation bias and computational overhead via ensemble methods, target network modifications, or distributional critics \cite{REDQ, TQC, CrossQ}, our approach takes a different path. By integrating logical objective composition via \lang into \algFull (\alg), we directly encode intended priorities, resulting in more efficient learning dynamics. This design enables us to achieve superior sample efficiency compared to state-of-the-art methods like CrossQ and TQC, without incurring additional computational cost.


\section{Definitions}
\label{sec:definitions}

\subsection{Multi-Objective Markov Decision Processes (MO-MDPs)}\label{def:mo_mdp}
MO-MDPs extend standard MDPs with vector-valued rewards.
It is defined as a tuple $\left(S, A, T, \vect{R}, \gamma\right)$ where:
\vspace{-1em}
\begin{table}[!h]
\centering
\begin{tabular}{r@{\hspace{3pt}}ll}
$S$ &: Set & State space \\
$A$ &: Set & Action space \\
$T$ &: $S \times A \rightarrow \Delta(S)$ & Transition distribution \\
$\vect{R}$ &: $S \times A \times S \rightarrow \mathbb{R}^n$ & Vector-valued reward function \\
$\gamma$ &: [0,1) & Discount factor \\
\end{tabular}
\end{table}
\vspace{-1.5em}

\noindent For MO-MDPs, the $\vect{Q}$ vector is defined as:
\vspace{-0.6em}
{
\footnotesize
\begin{equation}
    \vect{Q}^\pi(s,a) = \mathbb{E}_\pi \left[ \sum_{t=0}^{\infty} \gamma^t \vect{R}(s_t, a_t, s_{t+1}) \middle| \begin{array}{l}
        s_0 = s, a_0 = a, \\
        a_t \sim \pi(s_t),\\
        s_{t+1} \sim T(s_t, a_t)
    \end{array}\right]
    \label{eq:mo_mdp}
\end{equation}
}
\subsection{Fulfillment}\label{def:fulfillment}
A \textit{fulfillment} is any variable $f \in [0,1]$ that represents how much an objective is being fulfilled, where 0 means complete failure to fulfill and 1 is completely fulfilled. Intermediate values represent partial fulfillment.

\subsection{Multi-Fulfillment MDPs (MF-MDPs)}\label{def:mf_mdp}
We introduce MF-MDPs, which modify MO-MDPs by expressing each objective as a \emph{fulfillment}. We constrain the reward function such that $\vect{R}: S \times A \times S \rightarrow [0,1]^n$.
where each $R_i \in \vect{R}$ expresses the fulfillment of the $i^{th}$ objective. This formulation offers several key benefits:

\noindent\textbf{Comparable Objectives:} All objectives operate on the [0,1] scale, making their relative fulfillment directly comparable.

\noindent\textbf{Intuitive Composition:} fulfillment values can be composed using operations that preserve their semantic meaning (e.g., AND, OR operations), producing a new fulfillment value capturing their composite meaning.

\subsection{Fulfillment Q-values (FQ-values)}\label{def:fq}
In MF-MDPs, we can normalize Q-values becoming Fulfillment Q-values (FQ-values) due to the fact that the discounted sum in Eq.~\ref{eq:mo_mdp} is bounded from above by $1/(1-\gamma)$.
\vspace{-1em}
\begin{equation}
    \mathbf{FQ}^\pi(s,a) = (1-\gamma)\mathbf{Q}^\pi(s,a) \in [0,1]^n \quad \forall s \in S, a \in A
\vspace{-0.7em}
\end{equation}

These FQ-values indicate how well a policy $\pi$ expects to fulfill an objective over trajectories it would take in the MDP.
Practitioners must design these FQ-values to faithfully represent their intended fulfillment---\textbf{this correspondence is fundamental}, as improper distributions will inevitably distort objective prioritization and undermine the optimization process.
\section{\langFull{}}\label{sec:aps}

In multi-objective optimization, logical relationships among objectives (e.g., simultaneous satisfaction, alternatives, and priority) are often reduced to manual weight tuning that obscures intent.
We introduce \langFull{} (\lang{}), a domain-specific logic that uses power means to preserve these logical relationships.
power means interpolate between the minimum (for joint satisfaction or conservative composition) and the maximum (when any objective suffices or optimistic composition), enabling efficient gradient-based optimization.

\subsection{Power Mean Foundations}

The power mean $\pmean{p}$ unifies minimum and maximum through a continuous family of operators defined as:
\[ \pmean{p}(\vect{x}) = \pmean{p}((x_1, \ldots, x_n)) = \left(\frac{1}{n} \sum_{i=1}^n x_i^p\right)^{\frac{1}{p}}.\]

Depending on the value of the $p$ parameter, the power mean touches many well known operations:
{
\begin{align*}
p \to -\infty &: \text{Minimum} & p = 1 &: \text{Arithmetic mean} \\
p = -1 &: \text{Harmonic mean} & p = 2 &: \text{Root mean square} \\
p = 0 &: \text{Geometric mean} & p \to \infty &: \text{Maximum}
\end{align*}
}
Setting $p$ allows the specifier to transition between composition strategies smoothly. Generalizing linear utilities ($\pmean{1}$).

For any non-negative $\vect{x}$, i.e. such that $x_i \ge 0$, the power mean has the following properties \cite{power_mean_properties}:
\begin{align*}
\text{Range Preservation} \quad & \pmean{p}(\vect{x}) \in \left[\min(\vect{x}),\,\max(\vect{x})\right] \\
\text{Commutativity} \quad    & \pmean{p}(\vect{x}) = \pmean{p}(\text{permutation}(\vect{x})) \\
\text{Monotonicity} \quad & y \leq z \implies \pmean{p}((\vect{x}, y)) \leq \pmean{p}((\vect{x}, z)) \\
\text{Monotonicity in } p \quad & p_1 < p_2 \implies \pmean{p_1}(\vect{x}) \leq \pmean{p_2}(\vect{x})
\end{align*}

These properties ensure that composition behaves intuitively.
First, commutativity ensures the order of objectives doesn't matter.
Second, monotonicity guarantees that improving any objective only improves overall fulfillment.
Crucially, monotonicity in $p$ allows the power mean to smoothly interpolate between minimum and maximum, providing a continuous spectrum from pessimistic to optimistic composition. These properties allow deriving \lang's guarantees in \cref{sec:lang_guarantees}.

\subsection{Logical Language}
By building atop the power mean properties and formal guarantees in \cref{sec:lang_guarantees}, we introduce \lang{}. \lang is the first logic to formally express priority-aware objective composition specifications. When a specification is provided in \lang{}, the result is a formula $u: \text{\lang}\rightarrow [0, 1]$ that captures the desired compositional semantics as a utility function.

\noindent\textbf{(1) \lang Syntax:}
We define the syntax of \lang{} formulas using the following grammar:
\[
\begin{array}{lcl}
\phi &::=& f \mid \phi \pand{p} \phi \mid \phi \por{p} \phi \mid \neg \phi \mid [\phi]_\delta
\end{array}
\]
where:
\begin{itemize}
  \item $f \in [0,1]$ denotes a base fulfillment value;
  \item $p \leq 0$ in both $\pand{p}$ and $\por{p}$ operators;
  \item $\neg$ denotes logical negation;
  \item $[\phi]_\delta$ offsets the priority of $\phi$ by $\delta \in [-1,1]$.
\end{itemize}

\noindent\textbf{(2) Semantics:}
The semantics of \lang{} define how each operator transforms fulfillment values:
\begin{align}
u(f) &:= f&&\text{ for } f \in [0,1]\\
u(\phi_1 \pand{p} \phi_2) &:= \pmean{p}(u(\phi_1), u(\phi_2)) && \text{for } \phi_1, \phi_2 : \text{\lang}\\
u(\neg \phi) &:= 1 - u(\phi) && \text{for } \phi : \text{\lang}\\
u(\phi_1 \por{p} \phi_2) &:= u(\neg(\neg \phi_1 \pand{p} \neg \phi_2)) && \text{for } \phi_1, \phi_2 : \text{\lang}\\
u([\phi]_\delta) &:= \frac{u(\phi)+\max(\delta,0)}{1+\delta} && \text{for } \phi : \text{\lang}
\end{align}

These semantics preserve logical relationships between objectives. The conjunction $\pand{p}$ combines objectives that must be fulfilled together, with $p \leq 0$ controlling strictness (formalized in \cref{sec:lang_guarantees}). The disjunction $\por{p}$, defined via De Morgan's laws, allows focusing on the objective with highest marginal benefit. The offset $[\phi]_\delta$ creates priority by raising the baseline fulfillment of $\phi$, implementing a form of lexicographic ordering where objectives with lower offsets are prioritized. By making fulfillment composition explicit, \lang{} allows practitioners to seperate concerns and reason about the trade-offs between objectives in a principled way. We expect practitioners to first define fulfillment values for their objectives, testing whether their intention for the fulfillment of objectives maps well on to [0,1], and then use \lang{} to compose them in a way that is consistent with their priorities.

\subsection{Relation to Fuzzy Logic}\label{sec:aps_fuzzy_logic}
\lang{} similar to fuzzy logic, generalizes boolean operations to the continuous domain $[0,1]$, meaning, at the limits $\{0,1\}$ it is equivalent to boolean logic.
Choosing $p = 0$, our conjunction operator becomes the geometric mean ($\forall_{x, y \in [0,1]}, x \pand{0} y = \sqrt{x \cdot y}$) closely resembling the product t-norm \cite{tnorm} ($\forall_{x, y \in [0,1]}, x \land y = x\cdot y$) commonly used in fuzzy logic.
Also, when $p \to -\infty$, $\pand{p}$ becomes the minimum operator, which is equivalent to the fuzzy logic's minimum t-norm.

Contrary to fuzzy logic, which provides a notion of uncertainty, we emphasize fulfillment (\cref{def:fulfillment}). An illustrative example of the difference occurs when we compose a fuzzy variable with itself $x \pand{p} x$. In fuzzy logic this would evaluate to $x^2$, while in our framework it evaluates to $x$ (idempotence). Interpreted as fulfillment, $x^2$ would be considered less fulfilled contradicting our intuitive expectations.
Furthermore, \pmean{p} is not a t-norm it is not associative for every $p$, i.e.:\[%
\exists_{x, y, z, p}, \pmean{p}(\pmean{p}(x, y), z) \neq \pmean{p}(x, \pmean{p}(y, z)).
\]

\subsection{Relation to Hypervolume}
One of the most used multi-objective optimization metrics in the literature is the hypervolume indicator. It is defined as union of the volumes of the region dominated by the set of solutions. In the case with only one linear combination, the hypervolume becomes the product of the objectives being maximized. Existing techniques touting strong results \cite{xu2020prediction} maximize this metric which is equivalent to maximizing the \lang{} operator with $p=0$, even though these methods are recommending the use of linear utilities.

\subsection{Toy Example}
To demonstrate how \lang{} operators behave in practice, we present a minimal system with two fulfillment values $\fzero$ and $\fone$. These values compete for resources—improving one necessarily impacts the other—illustrating how our operators handle fundamental trade-offs.

\noindent\textbf{1) Multi-Objective Competitiveness:}
We derive each competing fulfillment value $f_i$ from a base fulfillment value $b_i \in [0,1]$. To model competition, each value directly reduces the other's fulfillment multiplicatively:
\begin{align}
    \fzero &= b_0(1 - \alpha b_1) \\
    \fone &= b_1(1 - \alpha b_0),
\end{align}
where $\alpha \in [0,1]$ controls competition strength. This creates natural tension—increasing either base value reduces the other's final fulfillment. We display the evolution of this optimization next to the operators.

\noindent\textbf{2) \lang Operator Effects:}
Using these competing fulfillment values, we examine how each operator resolves trade-offs:

\textsc{Conjunction ($\fzero \pand{p} \fone$) \inlinefig{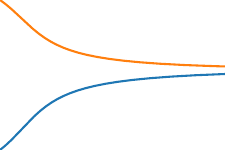}:}
In \lang, the intention to equally satisfy the two objectives can be expressed with the statement $\fzero \pand{p} \fone$. They converge to a compromise due to $\pand{p}$ assigning more importance to less fulfilled objectives.

\textsc{Disjunction ($\fzero \por{p} \fone$) \inlinefig{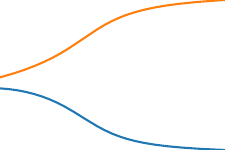}:}
When either value suffices, disjunction imposes that more fulfilled values are more important. The system rapidly maximizes $\fzero$---initialized slightly higher---at the cost of $\fone$.

\textsc{Priority Offset ($[\fzero]_\delta \pand{p} \fone$) \inlinefig{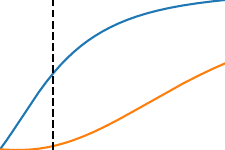}:}
The offset gives precedence to $\fone$ until it reaches the threshold established by $\delta$, at which point $\fzero$ becomes important as well.
This priority-based curriculum allows for coherent optimization, as opposed to methods that change the reward function requiring Q-values to be re-trained.

\noindent\textbf{3) Linear Utility ($\pmean{1}(\fzero, \fone)$) \inlinefig{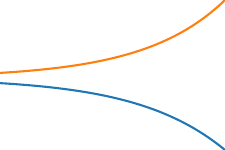}:}
Notice that the behavior of the linear utility is very similar to the disjunction with $p=1$, this is in stark contrast with the notion of convex combination that is usually associated with linearity. Specifically, under competitive dynamics, it no longer acts as a convex combination, but rather as a disjunction.

\subsection{FQ-value Composition}

In RL, we apply \lang{} operators to FQ-values (\cref{def:fq}) rather than immediate rewards. This allows reasoning about long-term trade-offs: an action might temporarily reduce one objective's fulfillment to achieve better overall fulfillment later. The $[0,1]$ bounds of FQ-values maintain our logical interpretation—if an objective has zero fulfillment, any conjunction involving it must also have zero fulfillment. These composed FQ-values then guide policy optimization while preserving the specified priorities between objectives.

\subsection{Guarantees in \lang}\label{sec:lang_guarantees}

In the context of \lang{}, $f_i \in [0,1]$ captures each objective's fulfillment level. The power mean enjoys one additional important property we term the \emph{minimum fulfillment bound}.

\begin{theorem}[Minimum Fulfillment Bound]\label{thm:f-bound}
\[\forall_{p \in \mathbb{R}, \vect{f} \in [0,1]^n}, \min(\vect{f}) \ge \sqrt[^p]{n((\pmean{p}(\vect{f}))^p - 1) + 1}.\]
\vspace{-1em}
\hrule
\vspace{0.3em}
This bound guarantees that when a power mean outputs value $y$, every input component must have at least fulfillment $\sqrt[^p]{n(y^p - 1) + 1}$.
\end{theorem}

\begin{proof}
Let $y = \pmean{p}(\vect{f})$. By \cref{lemma:worst-case}, there exists a $v \leq \min(\vect{f})$ such that $\pmean{p}((\vect{1}_{n-1}, v)) = y$. Then by \cref{lemma:explicit-solution}, we know $v = \sqrt[^p]{n(y^p - 1) + 1}$.
\end{proof}

This bound is crucial for practitioners providing concrete guarantees about objective fulfillment. For example, with two objectives ($n=2$) and $p=-2$, achieving an output of 0.9 guarantees each individual objective has at least 0.38 fulfillment. This minimum-fulfillment is particularly important in safety-critical applications where we need to ensure no objective is severely underperforming. The parameter $p$ allows practitioners to trade off between stronger guarantees (more negative $p$) and easier optimization (less negative $p$), as stricter bounds require more precise balancing of objectives.
~
\begin{theorem}[Power Mean Conservation]\label{thm:conservation}
\[\forall_{\vect{x}, \delta,i,j}\exists_{\delta'}, \pmean{p}(\vect{x}) = \pmean{p}(\vect{x} + \delta\mathbbm{1}_i - \delta'\mathbbm{1}_j)\]
\vspace{-1em}
\hrule
\vspace{0.3em}
where $\mathbbm{1}_i$ denotes a vector of zeros with a 1 in position $i$. This states that for any change $\delta$ to component $i$, there exists a change $\delta'$ to component $j$ that maintains the same power mean.
\end{theorem}

\begin{proof}
By commutativity of \pmean, for convenience and without loss of generality, we choose $i=0$ and $j=1$:
\begin{align*}
    \pmean{p}(\vect{x}) &= \pmean{p}(\vect{x} + \delta\mathbbm{1}_0 - \delta'\mathbbm{1}_1)\\
    x_0^p + x_1^p + \textstyle\sum_{k\ge 2} x_k^p &= (x_0 + \delta)^p + (x_1 - \delta')^p + \textstyle\sum_{k\ge 2} x_k^p\\
    (x_1 - \delta')^p &= x_0^p + x_1^p - (x_0 + \delta)^p\\
    \delta' &= x_1 - (x_0^p + x_1^p - (x_0 + \delta)^p)^{\frac{1}{p}}
\end{align*}
This proves the existencial by constructing $\delta'$ explicitly.
\end{proof}

\begin{lemma}[Worst Case Configuration]\label{lemma:worst-case}
\[\forall_{\vect{f} \in [0,1]^n} \exists_{v \in \mathbb{R}}: \pmean{p}((\vect{1}_{n-1}, v)) = \pmean{p}(\vect{f}) \text{ and } v \leq \min(\vect{f})\]
\vspace{-1em}
\hrule
\vspace{0.3em}
For any vector $\vect{f}$, there exists a vector with $n-1$ ones and a value $v$ that has the same power mean but with $v$ being bounded from above by the minimum of $\vect{f}$.
\end{lemma}

\begin{proof}
Let $m = \text{argmin}(\vect{f})$ be the index of the minimum component. By repeatedly invoking Theorem \cref{thm:conservation}, for each $i \in \{1,\ldots,n\} \setminus \{m\}$, we can increase component at $i$ to 1 while decreasing component $m$ to maintain the same power mean. Since we always decrease the minimum component or keep it the same, the final value $v$ of component $m$ must be less than $\min(\vect{f})$.
\end{proof}

\begin{lemma}[Explicit Minimum Solution]\label{lemma:explicit-solution}
For a vector with $n-1$ ones and one value $f$, if $\pmean{p}((\vect{1}_{n-1}, f)) = y$ then:
\vspace{-0.5em}
\[f = \sqrt[^p]{n(y^p - 1) + 1}\]
\begin{proof}
Let $\vect{x}$ be such a vector. Then:
\begin{align*}
y &= \left(\frac{1}{n} \sum_{i=1}^n x_i^p\right)^{\frac{1}{p}} \\
y^p &= \frac{1}{n} ((n-1) \cdot 1^p + f^p) \\
ny^p &= (n-1) + f^p \\
f^p &= n(y^p - 1) + 1 \\
f &= \sqrt[^p]{n(y^p - 1) + 1} \qedhere
\end{align*}
\end{proof}
\end{lemma}
\section{\algFull{} (\alg{})}\label{sec:bpg}

\algFull{} (\alg{}) extends Deep Deterministic Policy Gradient (DDPG) to efficiently optimize policies for \ourMDP{}s (\cref{def:mf_mdp}) using \lang{} specifications. The key innovation in \alg is its ability to directly accept and optimize for specifications written in \lang{}, bridging the gap between human-intuitive objective descriptions and reinforcement learning optimization. Unlike standard DDPG which operates on scalar Q-values, BPG works with \textcolor{blue}{Fulfillment Q-values ($\FQvalue$-values)} (\cref{def:fq}) that represent the degree to which each objective is fulfilled across time. These $\FQvalue$-values are then composed using the \textcolor{blue}{power mean operators as specified by the \lang{} formula}, preserving the logical relationships between objectives during policy updates. This approach enables the algorithm to make decisions that respect the intended priority and composition of objectives while maintaining the sample efficiency benefits of actor-critic methods.
\noindent\textit{We highlight in \textcolor{blue}{blue} our additions to DDPG.}


\subsection{Mitigating Overestimation Bias}

Overestimation bias presents a significant challenge in Q-learning based algorithms. Existing works such as REDQ~\cite{REDQ} and TQC~\cite{TQC} address this through more or larger critics. However, CrossQ~\cite{CrossQ}, contrary to the conclusions of REDQ\cite{REDQ}, hypothesizes that in standard MDP settings, overestimation bias does not affect sample efficiency, and indeed our experiments support the claims of CrossQ being more sample efficient. In \alg{}, however, this issue is particularly critical—inaccurate $\FQvalue$-value estimates present \lang{} with incorrect fulfillment values, leading to incorrect prioritization between objectives. Since \lang{} makes decisions based on the relative fulfillment levels of different objectives, even small estimation errors can significantly alter the learned behavior.

Our approach addresses this challenge through the addition of an \textcolor{blue}{observed discounted returns regularization}. For each rollout of length $n$, we compute observed fulfillment values:
\vspace{-0.5em}
\[\textcolor{blue}{\FVobs = (1-\gamma)\sum_{t=0}^{n-1} \gamma^t \vect{r}_t + \textsc{truncated}\cdot(1-\gamma)\vect{r}_n\frac{\gamma^n}{1-\gamma}}\vspace{-0.5em}\]
where $\FVobs$ is the observed fulfillment values, $\vect{r}_t$ are the normalized rewards at timestep $t$, $\gamma$ is the discount factor, and $\textsc{truncated}$ indicates episode truncation. This provides a conservative estimate of fulfillment-value, which we store alongside transition tuples in the replay buffer. Since $\FVobs$ represents returns from a previous policy with exploration noise, it serves as an effective underestimate that helps counteract overestimation without requiring additional critics.

\begin{figure}[tb]
\begin{algorithm}[H]
   \caption{\algFull{} (\alg{})}
\begin{algorithmic} \label{alg:bpg}
   \STATE Initialize networks and targets $\pi, \pi^\text{targ}$ and \textcolor{blue}{$\FQvalue, \FQvalue^\text{targ}$}
   \STATE Initialize replay buffer $B$
   \REPEAT
   \STATE Receive initial state $s_1$
   \FOR{each timestep $t$ in episode}
      \STATE $\theta^\pi \leftarrow \mathcal{N}(\theta^\pi, \sigma J^\text{prev})$ // performance-based noise
      \STATE $a_t \leftarrow \pi(s_t)$
      \STATE Execute $a_t$ and Store \textcolor{blue}{$\left(\vect{r}_t, s_t, a_t, s_{t+1}\right)$} in $B$
   \ENDFOR
   \STATE \textcolor{blue}{Compute and Store $\FVobs$ for each step in $B$}
      \FOR{each training iteration}
      \STATE $(s, a, \vect{r}, s_\text{next}, \textcolor{blue}{\FVobs}) \sim B$ // sample from the buffer
      \STATE $\vect{y}^\text{TD} \leftarrow \textcolor{blue}{(1-\gamma)}\vect{r} + \gamma$ \textcolor{blue}{$\FQvalue^\text{targ}$}$(s_\text{next}, \pi^\text{targ}(s_\text{next}))$
      \STATE $\LTD \leftarrow \pmean{2}(\vect{y}^\text{TD} -$ \textcolor{blue}{$\FQvalue$}$(s, a))$
      \STATE \textcolor{blue}{$\LFV \leftarrow \pmean{2}(\FVobs - \FQvalue(s, a))$}
      \STATE Update critic using $\textcolor{blue}{-\nabla_{\theta^\FQvalue}(\LTD + \alpha_{\text{FV}}\LFV)}$
      \STATE $J \leftarrow \textcolor{blue}{\pmean{2}(u_{\text{\lang}}(\FQvalue(s_i, \pi(s_i))))}$
         \STATE Update actor using policy gradient $\nabla_{\theta^\pi}J$
         \ENDFOR
      \STATE Update target networks
   \UNTIL{convergence}
\end{algorithmic}
\end{algorithm}
\vspace{-3em}
\end{figure}
\cref{alg:bpg} illustrates how these components work together. After collecting experience, we compute and store \textcolor{blue}{observed fulfillment values ($\FVobs$)} for each state-action pair. During training, we combine a standard temporal difference loss ($\LTD$) with a \textcolor{blue}{supervised loss against these observed values ($\LFV$)}. The \textcolor{blue}{power mean operators} are applied during policy updates, where \textcolor{blue}{$u_{\text{\lang}}(\FQvalue(s_i, a))$} scalarizes the vector-valued FQ-value according to the priority relationships specified in \lang{}. By working directly with fulfillment values in the [0,1] range, BPG ensures that the logical semantics of \lang{} operators are preserved throughout the learning process.

\section{Experiments}\label{sec:experiments}

We conducted a comprehensive empirical evaluation of \alg across multiple continuous control environments from the Farama-Foundation Gymnasium benchmark suite~\cite{towers2024gymnasium}.
Our experimental framework assesses two primary aspects: (1) sample efficiency, measured by the number of environment interactions required to reach predefined performance thresholds, and (2) the algorithm's robustness to overestimation bias through our normalization of value functions into $\FQvalue$-values.
We compared \alg against its baseline (DDPG~\cite{DDPG}) and several state-of-the-art reinforcement learning algorithms designed for sample efficiency, including SAC~\cite{SAC}, TQC~\cite{TQC}, and CrossQ~\cite{CrossQ}, to establish its relative performance characteristics. 
Importantly, while \alg is trained using our \lang framework—which structures rewards into prioritized objectives—we evaluate its performance using the original scalar rewards of the benchmark environments. This choice ensures our evaluation directly compares \alg against baselines on standard metrics, while demonstrating that our objective decomposition approach generalizes effectively to conventional performance measures.

\subsection{Performance on Benchmarks}
\subsubsection{Sample Efficiency}
Our results on several benchmark environments demonstrate significant improvements in sample efficiency compared to baseline and state-of-the-art methods. The top row in Fig.~\ref{fig:plots} summarizes these findings, showing substantial reductions in the number of steps required to reach target performance thresholds.

\begin{figure*}[ht]
    \centering
    \includegraphics[width=\textwidth]{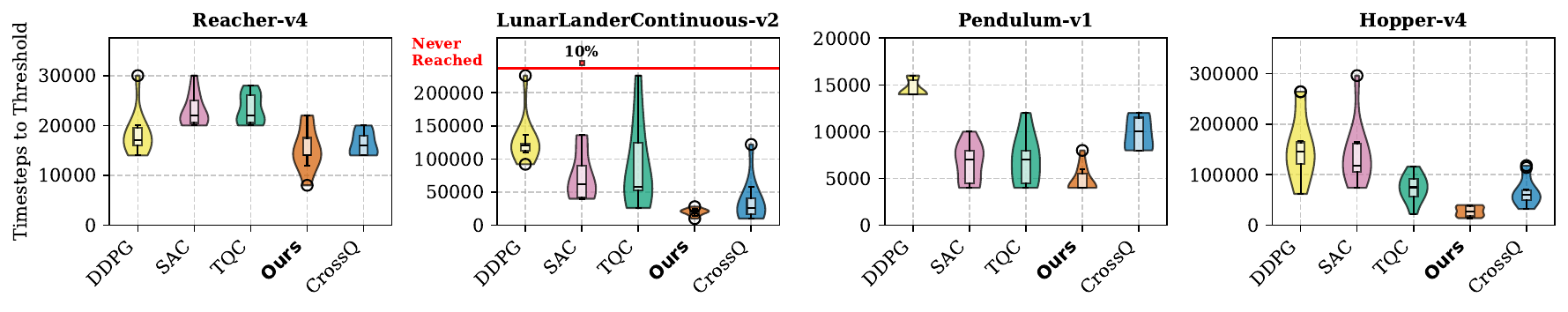}
    \includegraphics[width=\textwidth]{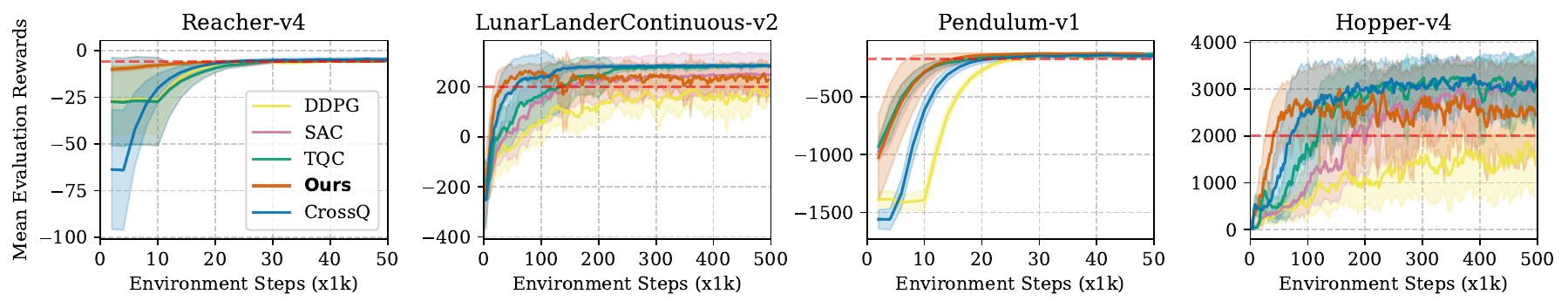}
    \vspace{-2em}
    \caption{ The top figures show violin plots indicating the distribution of timesteps required to reach performance thresholds accross 10 random seeds. The red horizontal line separates seeds failing to reach the threshold. In the bottom figures, we show a smoothened training progress of rewards versus environment steps for each algorithm. Shaded regions represent standard deviation accross seeds, and the dashed lines indicate the placement of reward thresholds for each environment.}
    \vspace{-1em}
    \label{fig:plots}
\end{figure*}

Our findings reveal substantial sample efficiency improvements across environments:

\noindent\textbf{LunarLanderContinuous-v2}: BPG reaches 200 rewards in 20,000 timesteps---84\% faster than DDPG (128,000) and 44\% faster than the state-of-the-art CrossQ (36,000).

\noindent\textbf{Hopper-v4}: BPG requires only 27,400 timesteps to reach the 2000 reward threshold, compared to 66,600 for CrossQ (59\% reduction) and 154,400 for DDPG (82\% reduction).

\noindent\textbf{Pendulum-v1} and \textbf{Reacher-v4}: BPG similarly outperforms other algorithms, with improvements of 51\% over CrossQ in Pendulum-v1 and outperforms all algorithms in Reacher-v4.

\subsubsection{Progress Plots}
Fig.~\ref{fig:plots} presents learning trajectories across environments, revealing two key advantages of BPG. First, BPG demonstrates significantly steeper learning curves, particularly in Pendulum-v1 and LunarLanderContinuous-v2, enabling rapid policy acquisition with minimal environment interactions. Second, BPG's learning curves show remarkable consistency, achieving near-monotonic improvement with rapidly increasing fulfillment. This suggests that \lang enables more coherent credit assignment during critical early learning stages, contributing to both accelerated initial learning and optimization stability throughout the training process.

\subsection{Overestimation Bias}\label{sec:experiment:overestimation}
In multi-objective settings, accurate Q-value estimation is crucial for proper objective prioritization. To evaluate BPG's resilience to overestimation bias, we conducted controlled experiments on the Hopper-v4 environment with deliberately reduced Polyak averaging.

Without our Q-value normalization mechanism $\FVobs$, the average Q-value error for Fulfillment Q-values (\cref{def:fq}) reached 0.627 after 38k steps. Adding underestimating loss with learning rate $\alpha_{\text{FV}}$=0.75 reduced error by 77\% to 0.146, while $\alpha_{\text{FV}}$=2.0 achieved similar results (0.138). This confirms that $\LFV$ mitigates overestimation bias without requiring additional critics or complex ensemble methods.


\subsection{Reward Engineering Comparison}
We compare our approach to traditional reward engineering methods, demonstrating how \lang{} simplifies the reward specification process while maintaining or improving performance. For each environment, we show the original reward function and our \lang{} specification, highlighting how the latter more clearly expresses the intended behavioral priorities. \textit{Note that we use $\vecand{p}{\vect{f}}$ to denote a vector of fulfillment values $\vect{f}$ being composed with the $\pand{p}$ operator.}
\subsubsection{Pendulum-v1}
On the left of the following table, we show the reward of Pendulum-v1, which is a weighted sum of angle and actuation terms with fine-tuned coefficients
\vspace{-1em}
\begin{table}[!h]
\centering
\setlength{\intextsep}{0pt}
\begin{tabular}{>{\centering\arraybackslash}p{0.22\textwidth}>{\centering\arraybackslash}p{0.22\textwidth}}
\midrule
\textbf{Pendulum-v1 Reward Function} & \textbf{\lang Specification:} $\phi_{\text{pendulum}}$ \\
\midrule
$-\theta^2 - 0.1\,\dot{\theta}^2 - 0.001\,\text{torque}^2$ & $F_{\text{angle}}^2 \pand{p} F_{\text{actuation}}$ \\
\midrule
\end{tabular}
\end{table}
\vspace{-1.3em}\\
On the right, we show our \lang{} specification. Here $F_{\text{angle}}$ is the fulfillment value for angle alignment, and $F_{\text{actuation}}$ is represents minimizing actuation fullfillment. The squared angle term emphasizes the primary task of angle alignment.

\subsubsection{Reacher-v4}
The reward is described by a fine-tuned weighted sum of distance and the norm of torque terms
\vspace{-1em}
\begin{table}[!h]
\centering
\setlength{\tabcolsep}{4pt}
\begin{tabular}{>{\centering\arraybackslash}p{0.22\textwidth}>{\centering\arraybackslash}p{0.22\textwidth}}
\midrule
\textbf{Reacher-v4 Reward Function} & \textbf{\lang Specification:} $\phi_{\text{reacher}}$ \\
\midrule
$-\text{distance} - 0.1||torque||^2$ & $F_{\text{distance}}^2 \pand{p} \vecand{p}{\vect{F}_{\text{torque}}}$ \\
\midrule
\end{tabular}
\end{table}
\vspace{-1.3em}\\
Our \lang specifcation represents reaching the target with $F_{\text{distance}}$, squared for emphasis, and minimizing the torque fulfillments with $\vecand{p}{\vect{F}_{\text{torque}}}$. .

\subsubsection{Hopper-v4}
The reward is described by a fine-tuned weighted sum of velocity and the norm of action terms
\vspace{-1em}
\begin{table}[!h]
\centering
\setlength{\tabcolsep}{4pt}
\begin{tabular}{>{\centering\arraybackslash}p{0.22\textwidth}>{\centering\arraybackslash}p{0.22\textwidth}}
\midrule
\textbf{Hopper-v4 Reward Function} & \textbf{\lang Specification:} $\phi_{\text{hopper}}$ \\
\midrule
$1 + \frac{dx}{dt} - 0.001 \cdot ||action||^2_2$ & $\vecand{p}{\vect{F}_{\text{speed}}} \pand{p} \vecand{p}{\vect{F}_{\text{action}}}$ \\
\midrule
\end{tabular}
\end{table}
\vspace{-1.3em}\\
Here $\vecand{p}{\vect{F}_{\text{speed}}}$ represents the fulfillments for the velocity of each limb in the Hopper, and $\vecand{p}{\vect{F}_{\text{action}}}$ represents the fulfillments of the minimizing the three joint torques.

\subsubsection{LunarLanderContinuous-v2}
The original reward in LunarLander is particularly complex, defined as:
\vspace{-0.7em}
\begin{align*}
&\text{distance\_reward} + \text{velocity\_reward} + \text{angle\_reward} \\
&+ 10 \cdot \text{legs\_contact} - 0.3 \cdot |main\_engine| \\
&- 0.03 \cdot |side\_engines| + \text{terminal\_reward}
\end{align*}\\[-1.7em]
where the distance, velocity, and angle rewards increase as the lander gets closer to the landing pad, moves slower, and stays more horizontal. The terminal reward is +100 for safe landing or -100 for crashing.

Our \lang{} specification $\phi_{\text{lander}}$ uses a hierarchical structure:
\vspace{-1.5em}
\begin{align*}
\vecand{p}{{F_{\text{near}}, [F_{\text{very\_near}}]_{0.1}, [F_{\text{legs}}]_{0.1}, [F_{\text{landed}}]_{0.1}, [F_{\text{fuel}}]_{0.5}}}
\end{align*}\\[-1.5em]
Here the offsets create a natural curriculum during training: the agent first focuses on basic proximity ($F_{\text{near}}$), then simultaneously addresses precise positioning, leg contact, and landing ($[F_{\text{very\_near}}]_{0.1}$, $[F_{\text{legs}}]_{0.1}$, $[F_{\text{landed}}]_{0.1}$), and finally optimizes fuel efficiency ($[F_{\text{fuel}}]_{0.5}$) once the primary landing objectives are reasonably satisfied. The conjunction ensures all objectives must ultimately be satisfied for successful landing.

\subsection{Behavioral Analysis}
Standard reward functions often embody fundamental limitations that \lang{} effectively addresses. In LunarLander, the non-Markovian reward aggregates multiple state-history components, complicating Q-value estimation and impeding learning efficiency. Hopper-v4 exemplifies semantic ambiguity, where identical reward values ($\sim$1000) can represent qualitatively distinct behaviors---either sustained upright posture without progression or significant forward motion lacking stability---conflating disparate policy qualities. In contrast, the learned behavior specified by \lang{} distinguishes between these behaviors, as numerically evaluated in \cref{sec:exp:ablation}.

\textit{A note on parameter selection in \lang{}:}
\lang{} is robust to reasonable variations in power mean parameters and offsets, we choose $p$ as either 0 or $-1$, which primarily serve to optimize sample efficiency rather than fundamentally changing the desired behavior. For example, not squaring the angle term in Pendulum would still result in an upright pendulum, but with slower convergence due to more conservative actions. This behavioral consistency persists across training runs, unlike linear weighted reward functions that often converge to different local optima depending on initialization.

\subsection{Ablation Study: Impact of \lang on Behavior}\label{sec:exp:ablation}
\vspace{-1em}
\begin{table}[H]
\centering
\label{tab:aps-ablation}
\begin{tabular}{p{0.15\textwidth}cc}
\hline
\textbf{Metric} & \textbf{With \lang} & \textbf{Without \lang} \\
\hline
\lang:$\phi_{hopper}$ & 0.625 & 0.194 \\
Hopper-v4 Reward & 2288.80 & 750.35 \\
\hline
\end{tabular}
\end{table}\vspace{-1em}
In the table above, we describe BPG's performance in Hopper-v4 on 10 seeds after 48k steps of training. Beyond raw performance gains with \lang, we observed a critical qualitative difference: without \lang, agents frequently achieved rewards of approximately 1000 by simply standing still---a reward hacking scenario where linear rewards where fulfilled in the original reward function but failed to achieve the intended behavior. Our \lang formulation assigned near-zero fulfillment values ($3.8 \times 10^{-5}$) to such behaviors, correctly identifying them as failing to satisfy the intended objectives as the agent must move all ($\vecand{p}{\vect{F}_{\text{speed}}}$) the limbs forward to be considered fulfilled.
Our \lang specification generalized effectively across multiple MuJoCo locomotion environments (HalfCheetah-v4, Walker2d-v4, Ant-v4), consistently producing forward progression in preliminary tests.

\section{Limitations and Future Work}\label{sec:limitations}
This paper introduced \langFull (\lang), bridging the intent-to-behavior gap in multi-objective reinforcement learning through power-mean operators over normalized objectives.
Our \algFull algorithm achieves state-of-the-art sample efficiency while preserving intended behavioral priorities.
Despite these advances, \alg shows lower asymptotic performance, and \lang exhibits sensitivity to overestimation bias that our FQ-value estimation addresses but requires an additional hyperparameter.

Future work should analyze gradient propagation through \lang operators, address replay buffer implementation complexities, and integrate \lang with more sophisticated algorithms to combine sample efficiency with improved convergence properties. Extensions to handle dynamic objectives would further enhance \lang's applicability to long-horizon problems with changing constraints.

\bibliographystyle{ieeetr}
\bibliography{references}


\end{document}